\pdfoutput=1
\def\year{2020}\relax

\documentclass[letterpaper]{article} 
\usepackage{aaai20}  
\usepackage{times}  
\usepackage{helvet} 
\usepackage{courier}  
\usepackage[hyphens]{url}  
\usepackage{graphicx} 
\usepackage{graphicx}  
\usepackage{amsmath}
\usepackage{amsfonts}
\usepackage{amsthm}
\usepackage{booktabs}
\usepackage{subcaption}
\usepackage{dsfont}
\usepackage{multirow}
\usepackage{enumerate}
\usepackage{pifont}

\newcommand{\citet}[1]{\citeauthor{#1}~\shortcite{#1}}
\newcommand{\citep}{\cite}

\newtheorem{prop}{Proposition}
\renewcommand{\mod}{\text{mod}\,}

\nocopyright
\title{Learning Hierarchy-Aware Knowledge Graph Embeddings for Link Prediction}
\author{Zhanqiu Zhang,\thanks{Equal contribution.} Jianyu Cai,\footnotemark[1] Yongdong Zhang, Jie Wang\thanks{Corresponding author.}\\
University of Science and Technology of China\\ 
\{zqzhang, jycai\}@mail.ustc.edu.cn\\
\{zhyd73, jiewangx\}@ustc.edu.cn
}
 
\begin{document}

\maketitle

\begin{abstract}
Knowledge graph embedding, which aims to represent entities and relations as low dimensional vectors (or matrices, tensors, etc.), has been shown to be a powerful technique for predicting missing links in knowledge graphs. Existing knowledge graph embedding models mainly focus on modeling relation patterns such as symmetry/antisymmetry, inversion, and composition. However, many existing approaches fail to model \textit{semantic hierarchies}, which are common in real-world applications. To address this challenge, we propose a novel knowledge graph embedding model---namely, \textbf{H}ierarchy-\textbf{A}ware \textbf{K}nowledge Graph \textbf{E}mbedding (HAKE)---which maps entities into the polar coordinate system. HAKE is inspired by the fact that concentric circles in the polar coordinate system can naturally reflect the hierarchy. Specifically, the radial coordinate aims to model entities at different levels of the hierarchy, and entities with smaller radii are expected to be at higher levels; the angular coordinate aims to distinguish entities at the same level of the hierarchy, and these entities are expected to have roughly the same radii but different angles. Experiments demonstrate that HAKE can effectively model the semantic hierarchies in knowledge graphs, and significantly outperforms existing state-of-the-art methods on benchmark datasets for the link prediction task. 
\end{abstract}

\section{Introduction}

Knowledge graphs are usually collections of factual triples---(head entity, relation, tail entity), which represent human knowledge in a structured way. In the past few years, we have witnessed the great achievement of knowledge graphs in many areas, such as natural language processing \citep{ernie}, question answering \citep{KGQA}, and recommendation systems \citep{KGRS}.

Although commonly used knowledge graphs contain billions of triples, they still suffer from the incompleteness problem that a lot of valid triples are missing, as it is impractical to find all valid triples manually. Therefore, knowledge graph completion, also known as link prediction in knowledge graphs, has attracted much attention recently. Link prediction aims to automatically predict missing links between entities based on known links. It is a challenging task as we not only need to predict whether there is a relation between two entities, but also need to determine which relation it is.

Inspired by word embeddings \citep{word2vec} that can well capture semantic meaning of words, researchers turn to distributed representations of knowledge graphs (aka, knowledge graph embeddings) to deal with the link prediction problem. Knowledge graph embeddings regard entities and relations as low dimensional vectors (or matrices, tensors), which can be stored and computed efficiently. Moreover, like in the case of word embeddings, knowledge graph embeddings can preserve the semantics and inherent structures of entities and relations. Therefore, other than the link prediction task, knowledge graph embeddings can also be used in various downstream tasks, such as triple classification \citep{transr}, relation inference \citep{RSN}, and search personalization \citep{capse}.

The success of existing knowledge graph embedding models heavily relies on their ability to model connectivity patterns of the relations, such as symmetry/antisymmetry, inversion, and composition  \citep{rotate}. For example, TransE \citep{transe},
which represent relations as translations, can model the inversion and composition patterns. DistMult \citep{distmult}, which models the three-way interactions between head entities, relations, and tail entities, can model the symmetry pattern. RotatE \citep{rotate}, which represents entities as points in a complex space and relations as rotations, can model relation patterns including symmetry/antisymmetry, inversion, and composition.
However, many existing models fail to model \textit{semantic hierarchies} in knowledge graphs.

Semantic hierarchy is a ubiquitous property in knowledge graphs. For instance, WordNet \citep{wordnet} contains the triple [arbor/cassia/palm, hypernym, tree], where ``tree'' is at a higher level than ``arbor/cassia/palm'' in the hierarchy. Freebase \citep{freebase} contains the triple [England, /location/location/contains, Pontefract/Lancaster], where ``Pontefract/Lancaster'' is at a lower level than ``England'' in the hierarchy. Although there exists some work that takes the hierarchy structures into account \citep{hType,hRel}, they usually require additional data or process to obtain the hierarchy information. Therefore, it is still challenging to find an approach that is capable of modeling the semantic hierarchy automatically and effectively.

In this paper, we propose a novel knowledge graph embedding model---namely, \textbf{H}ierarchy-\textbf{A}ware \textbf{K}nowledge Graph \textbf{E}mbedding (HAKE). To model the semantic hierarchies, HAKE is expected to distinguish entities in two categories: (a) at different levels of the hierarchy; (b) at the same level of the hierarchy. Inspired by the fact that entities that have the hierarchical properties can be viewed as a tree, we can use the depth of a node (entity) to model different levels of the hierarchy. Thus, we use modulus information to model entities in the category (a), as the size of moduli can reflect the depth. Under the above settings, entities in the category (b) will have roughly the same modulus, which is hard to distinguish. Inspired by the fact that the points on the same circle can have different phases, we use phase information to model entities in the category (b). Combining the modulus and phase information, HAKE maps entities into the polar coordinate system, where the radial coordinate corresponds to the modulus information and the angular coordinate corresponds to the phase information.
Experiments show that our proposed HAKE model can not only clearly distinguish the semantic hierarchies of entities, but also significantly and consistently outperform several state-of-the-art methods on the benchmark datasets.

\vspace{3mm}
\noindent
\textbf{Notations} Throughout this paper, we use lower-case letters $h$, $r$, and $t$ to represent head entities, relations, and tail entities, respectively. The triplet $(h,r,t)$ denotes a fact in knowledge graphs. The corresponding boldface lower-case letters $\textbf{h}$, $\textbf{r}$ and $\textbf{t}$ denote the embeddings (vectors) of head entities, relations, and tail entities. The $i$-th entry of a vector $\textbf{h}$ is denoted as $[\textbf{h}]_i$. 
Let $k$ denote the embedding dimension.

Let $\circ:\mathbb{R}^n\times\mathbb{R}^n\rightarrow\mathbb{R}^n$ denote the Hadamard product between two vectors, that is,
\begin{align*}
    [\textbf{a}\circ \textbf{b}]_i=[\textbf{a}]_i\cdot [\textbf{b}]_i,
\end{align*}
and $\|\cdot\|_1$, $\|\cdot\|_2$ denote the $\ell_1$ and $\ell_2$ norm, respectively.

\section{Related Work}
In this section, we will describe the related work and the key differences between them and our work in two aspects---the model category and the way to model hierarchy structures in knowledge graphs.

\subsection{Model Category}
Roughly speaking, we can divide knowledge graph embedding models into three categories---translational distance models, bilinear models, and neural network based models. Table \ref{table:related_works} exhibits several popular models.

\textbf{Translational distance models} describe relations as translations from source entities to target entities. TransE \citep{transe} supposes that entities and relations satisfy $\textbf{h}+\textbf{r}\approx \textbf{t}$, where $\textbf{h}, \textbf{r}, \textbf{t} \in \mathbb{R}^n$, and defines the corresponding score function as $f_r(\textbf{h},\textbf{t})=-\|\textbf{h}+\textbf{r}-\textbf{t}\|_{1/2}$. However, TransE does not perform well on 1-N, N-1 and N-N relations \citep{transh}. 
TransH \citep{transh} overcomes the many-to-many relation problem by allowing entities to have distinct representations given different relations. 
The score function is defined as $f_r(\textbf{h},\textbf{t})=-\|\textbf{h}_{\perp}+\textbf{r}-\textbf{t}_{\perp}\|_2$, where $\textbf{h}_{\perp}$ and $\textbf{t}_{\perp}$ are the projections of entities onto relation-specific hyperplanes. 
ManifoldE \citep{manifolde} deals with many-to-many problems by relaxing the hypothesis $\textbf{h}+\textbf{r}\approx \textbf{t}$ to $\|\textbf{h}+\textbf{r}-\textbf{t}\|_2^2\approx\theta_r^2$ for each valid triple. In this way, the candidate entities can lie on a manifold instead of exact point. The corresponding score function is defined as $f_r(\textbf{h},\textbf{t})=-(\|\textbf{h}+\textbf{r}-\textbf{t}\|_2^2-\theta_r^2)^2$.
More recently, to better model symmetric and antisymmetric relations, RotatE \citep{rotate} defines each relation as a rotation from source entities to target entities in a complex vector space. The score function is defined as $f_r(\textbf{h},\textbf{t})=-\|\textbf{h}\circ \textbf{r}-\textbf{t}\|_1$, where $\textbf{h},\textbf{r},\textbf{t}\in\mathbb{C}^k$ and $|[\textbf{r}]_i|=1$.

\begin{table*}[!ht]
    \caption{Details of several knowledge graph embedding models, where $\circ$ denotes the Hadamard product, $f$ denotes a activation function, $*$ denotes 2D convolution, and $\omega$ denotes a filter in convolutional layers. $\bar{\cdot}$ denotes conjugate for complex vectors in ComplEx model and 2D reshaping for real vectors in ConvE model.}
    \centering
    \resizebox{2.1\columnwidth}{!}{
    \begin{tabular}{ccc}
        \toprule
        \textbf{Model} & \textbf{Score Function $f_r(\textbf{h},\textbf{t})$} & \textbf{Parameters}\\
        \midrule
        TransE \citep{transe}    & $-\|\textbf{h}+\textbf{r}-\textbf{t}\|_{1/2}$   & $\textbf{h}, \textbf{r}, \textbf{t}\in\mathbb{R}^k$\\
        TransR \citep{transr}     & $-\|\textbf{M}_r\textbf{h}+\textbf{r}-\textbf{M}_r\textbf{t}\|_2$ & $\textbf{h}, \textbf{t}\in\mathbb{R}^d$, $r\in\mathbb{R}^k$, $\textbf{M}_r\in\mathbb{R}^{k\times d}$\\
        ManifoldE \citep{manifolde}   & $-(\|\textbf{h}+\textbf{r}-\textbf{t}\|_2^2-\theta_r^2)^2$    & $\textbf{h}, \textbf{r}, \textbf{t}\in\mathbb{R}^k$\\
        RotatE \citep{rotate}     &$-\|\textbf{h}\circ \textbf{r}-\textbf{t}\|_2$  & $\textbf{h}, \textbf{r}, \textbf{t}\in\mathbb{C}^k$, $|r_i|=1$\\
        \midrule
        RESCAL \citep{rescal}     &$\textbf{h}^\top \textbf{M}_r\textbf{t}$    & $\textbf{h}, \textbf{t}\in\mathbb{R}^k$, $\textbf{M}_r\in\mathbb{R}^{k\times k}$\\
        DistMult \citep{distmult} &$\textbf{h}^\top \text{diag}(\textbf{r})\textbf{t}$    &$\textbf{h}, \textbf{r}, \textbf{t}\in\mathbb{R}^k$\\
        ComplEx \citep{complex}   &$\text{Re}( \textbf{h}^\top\text{diag}(\textbf{r})\bar{\textbf{t}})$ & $\textbf{h}, \textbf{r}, \textbf{t}\in\mathbb{C}^k$ \\
        \midrule
        ConvE \citep{conve} &$f(\text{vec}(f([\bar{\textbf{r}},\bar{\textbf{h}}]*\omega))\textbf{W})\textbf{t}$    & $\textbf{h}, \textbf{r}, \textbf{t}\in\mathbb{R}^k$\\
        \midrule
        \multirow{2}*{HAKE}    & \multirow{2}*{$-\|\textbf{h}_m\circ \textbf{r}_m-\textbf{t}_m\|_2-\lambda\|\sin((\textbf{h}_p+\textbf{r}_p-\textbf{t}_p)/2)\|_1$} &$\textbf{h}_m, \textbf{t}_m\in\mathbb{R}^k$,$\textbf{r}_m\in\mathbb{R}_+^k$, \\ &&$\textbf{h}_p, \textbf{r}_p, \textbf{t}_p\in[0,2\pi)^k$, $,\lambda\in\mathbb{R}$\\
        \bottomrule
    \end{tabular}
    }
    \label{table:related_works}
\end{table*}

\textbf{Bilinear models} product-based score functions to match latent semantics of entities and relations embodied in their vector space representations. RESCAL \citep{rescal} represents each relation as a full rank matrix, and defines the score function as $f_r(\textbf{h},\textbf{t})=\textbf{h}^\top \textbf{M}_r \textbf{t}$, which can also be seen as a bilinear function. As full rank matrices are prone to overfitting, recent works turn to make additional assumptions on $\textbf{M}_r$. For example, DistMult \citep{distmult} assumes $\textbf{M}_r$ to be a diagonal matrix, and ANALOGY \citep{analogy} supposes that $\textbf{M}_r$ is normal. However, these simplified models are usually less expressive and not powerful enough for general knowledge graphs. Differently, ComplEx \citep{complex} extends DistMult by introducing complex-valued embeddings to better model asymmetric and inverse relations. HolE \citep{hole} combines the expressive power of RESCAL with the efficiency and simplicity of DistMult by using the circular correlation operation.

\textbf{Neural network based models} have received greater attention in recent years. For example, MLP \citep{mlp} and NTN \citep{ntn} use a fully connected neural network to determine the scores of given triples.  ConvE \citep{conve} and ConvKB \citep{convkb} employ convolutional neural networks to define score functions. Recently, graph convolutional networks are also introduced, as knowledge graphs obviously have graph structures \citep{rgcn}. 

Our proposed model HAKE belongs to the translational distance models. More specifically, HAKE shares similarities with RotatE \citep{rotate}, in which the authors claim that they use both modulus and phase information. However, there exist two major differences between RotatE and HAKE. Detailed differences are as follows.
\begin{enumerate}[(a)]
    \item The aims are different. RotatE aims to model the relation patterns including symmetry/antisymmetry, inversion, and composition. HAKE aims to model the semantic hierarchy, while it can also model all the relation patterns mentioned above.
    \item The ways to use modulus information are different. RotatE models relations as rotations in the complex space, which encourages two linked entities to have the same modulus, no matter what the relation is. The different moduli in RotatE come from the inaccuracy in training. Instead, HAKE explicitly models the modulus information, which significantly outperforms RotatE in distinguishing entities at different levels of the hierarchy.
\end{enumerate} 

\subsection{The Ways to Model Hierarchy Structures}
Another related problem is how to model hierarchy structures in knowledge graphs. Some recent work considers the problem in different ways. \citet{hCat} embed entities and categories jointly into a semantic space and designs models for the concept categorization and dataless hierarchical classification tasks. \citet{hRel} use clustering algorithms to model the hierarchical relation structures. \citet{hType} proposed TKRL, which embeds the type information into knowledge graph embeddings. That is, TKRL requires additional hierarchical type information for entities.  

Different from the previous work, our work
\begin{enumerate}[(a)]
    \item considers the link prediction task, which is a more common task for knowledge graph embeddings;
    \item can automatically learn the semantic hierarchy in knowledge graphs without using clustering algorithms; 
    \item does not require any additional information other than the triples in knowledge graphs. 
\end{enumerate}

\section{The Proposed HAKE}
In this section, we introduce our proposed model HAKE. We first introduce two categories of entities that reflect the semantic hierarchies in knowledge graphs. Afterwards, we introduce our proposed HAKE that can model entities in both of the categories.

\subsection{Two Categories of Entities}
To model the semantic hierarchies of knowledge graphs, a knowledge graph embedding model must be capable of distinguishing entities in the following two categories.
\begin{enumerate}[(a)]
    \item Entities at different levels of the hierarchy. For example, ``mammal'' and ``dog'', ``run'' and ''move''.
    \item Entities at the same level of the hierarchy. For example, ``rose'' and ``peony'', ``truck'' and ''lorry''.
\end{enumerate}

\subsection{Hierarchy-Aware Knowledge Graph Embedding}
To model both of the above categories, we propose a hierarchy-aware knowledge graph embedding model---HAKE. HAKE consists of two parts---the modulus part and the phase part---which aim to model entities in the two different categories, respectively. Figure \ref{fig:mode_illustration} gives an illustration of the proposed model.

To distinguish embeddings in the different parts, we use $\textbf{e}_m$ ($\textbf{e}$ can be $\textbf{h}$ or $\textbf{t}$) and $\textbf{r}_m$ to denote the entity embedding and relation embedding in the modulus part, and use $\textbf{e}_p$ ($\textbf{e}$ can be $\textbf{h}$ or $\textbf{t}$) and $\textbf{r}_p$ to denote the entity embedding and relation embedding in the phase part.

\textbf{The modulus part} aims to model the entities at different levels of the hierarchy. Inspired by the fact that entities that have hierarchical property can be viewed as a tree, we can use the depth of a node (entity) to model different levels of the hierarchy. Therefore, we use modulus information to model entities in the category (a), as moduli can reflect the depth in a tree. Specifically, we regard each entry of $\textbf{h}_m$ and $\textbf{t}_m$, that is, $[\textbf{h}_m]_i$ and $[\textbf{t}_m]_i$, as a modulus, and regard each entry of $\textbf{r}_m$, that is, $[\textbf{r}]_i$, as a scaling transformation between two moduli. We can formulate the modulus part as follows:
\begin{align*}
    \textbf{h}_m\circ \textbf{r}_m=\textbf{t}_m, \text{ where }\,\textbf{h}_m, \textbf{t}_m\in\mathbb{R}^k, \text{ and } \textbf{r}_m\in\mathbb{R}_+^k.
\end{align*}
The corresponding distance function is:
\begin{align*}
    d_{r,m}(\textbf{h}_m, \textbf{t}_m)=\|\textbf{h}_m\circ \textbf{r}_m-\textbf{t}_m\|_2.
\end{align*}

\begin{figure}[ht]
  \centering 
  \includegraphics[width=0.9\columnwidth]{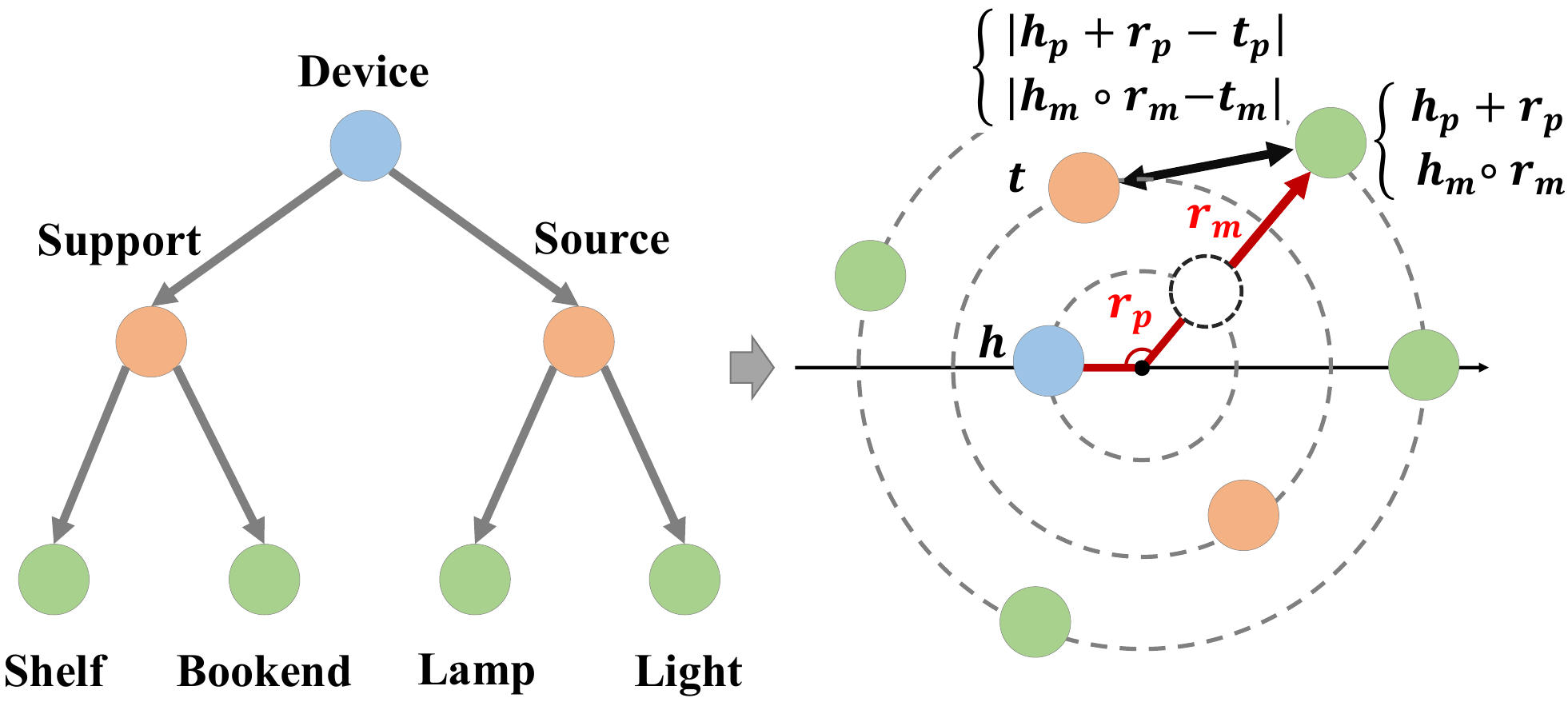}
\caption{Simple illustration of HAKE. In a polar coordinate system, the radial coordinate aims to model entities at different levels of the hierarchy, and the angular coordinate aims to distinguish entities at the same level of the hierarchy.}
\label{fig:mode_illustration}
\end{figure}

Note that we allow the entries of entity embeddings to be negative but restrict the entries of relation embeddings to be positive. This is because that the signs of entity embeddings can help us to predict whether there exists a relation between two entities. For example, if there exists a relation $r$ between $h$ and $t_1$, and no relation between $h$ and $t_2$, then $(h, r, t_1)$ is a positive sample and $(h, r, t_2)$ is a negative sample. Our goal is to minimize $d_r(\textbf{h}_m, \textbf{t}_{1,m})$ and maximize $d_r(\textbf{h}_m, \textbf{t}_{2,m})$, so as to make a clear distinction between positive and negative samples. For the positive sample, $[\textbf{h}]_i$ and $[\textbf{t}_1]_i$ tend to share the same sign, as $[\textbf{r}_m]_i>0$. For the negative sample, the signs of $[\textbf{h}_m]_i$ and $[\textbf{t}_{2,m}]_i$ can be different if we initialize their signs randomly. In this way, $d_r(\textbf{h}_m, \textbf{t}_{2,m})$ is more likely to be larger than $d_r(\textbf{h}_m, \textbf{t}_{1,m})$, which is exactly what we desire. We will validate this argument by experiments in Section $4$ of the supplementary material.

Further, we can expect the entities at higher levels of the hierarchy to have smaller modulus, as these entities are more close to the root of the tree.

If we use only the modulus part to embed knowledge graphs, then the entities in the category (b) will have the same modulus. Moreover, suppose that $r$ is a relation that reflects the same semantic hierarchy, then $[\textbf{r}]_i$ will tend to be one, as $h\circ r\circ r=h$ holds for all $h$. Hence, embeddings of the entities in the category (b) tend to be the same, which makes it hard to distinguish these entities. Therefore, a new module is required to model the entities in the category (b). 

\textbf{The phase part} aims to model the entities at the same level of the semantic hierarchy. Inspired by the fact that points on the same circle (that is, have the same modulus) can have different phases, we use phase information to distinguish entities in the category (b). Specifically, we regard each entry of $\textbf{h}_p$ and $\textbf{t}_p$, that is, $[\textbf{h}_p]_i$ and $[\textbf{t}_p]_i$ as a phase, and regard each entry of $\textbf{r}_p$, that is, $[\textbf{r}_p]_i$, as a phase transformation. We can formulate the phase part as follows:
\begin{align*}
    (\textbf{h}_p+\textbf{r}_p)\text{mod}\, 2\pi=\textbf{t}_p, \text{where } \textbf{h}_p,\textbf{r}_p,\textbf{t}_p\in[0,2\pi)^k.
\end{align*}
The corresponding distance function is:
\begin{align*}
    d_{r,p}(\textbf{h}_p,\textbf{t}_p)=\|\sin((\textbf{h}_p+\textbf{r}_p-\textbf{t}_p)/2)\|_1,
\end{align*}
where $\sin(\cdot)$ is an operation that applies the sine function to each element of the input. Note that we use a sine function to measure the distance between phases instead of using $\|\textbf{h}_p+\textbf{r}_p-\textbf{t}_p\|_1$, as phases have periodic characteristic. This distance function shares the same formulation with that of pRotatE \citep{rotate}.

Combining the modulus part and the phase part, HAKE maps entities into the \textbf{polar coordinate system}, where the radial coordinate and the angular coordinates correspond to the modulus part and the phase part, respectively. That is, HAKE maps an entity $h$ to $[\textbf{h}_m;\textbf{h}_p]$, where $\textbf{h}_m$ and $\textbf{h}_p$ are generated by the modulus part and the phase part, respectively, and $[\,\cdot\,; \,\cdot\,]$ denotes the concatenation of two vectors. Obviously, $([\textbf{h}_m]_i,[\textbf{h}_p]_i)$ is a 2D point in the polar coordinate system. Specifically, we formulate HAKE as follows:
\begin{align*}
    \begin{cases}
    \textbf{h}_m\circ \textbf{r}_m=\textbf{t}_m, \text{ where } {\textbf{h}_m,\textbf{t}_m\in\mathbb{R}^k,\textbf{r}_m}\in\mathbb{R}_+^k,\\
    (\textbf{h}_p+\textbf{r}_p)\text{mod}\, 2\pi=\textbf{t}_p, \text{ where } {\textbf{h}_p,\textbf{t}_p,\textbf{r}_p}\in [0, 2\pi)^k.
    \end{cases}
\end{align*}
The distance function of HAKE is:
\begin{align*}
    d_{r}(\textbf{h},\textbf{t})=d_{r,m}(\textbf{h}_m,\textbf{t}_m)+\lambda d_{r,p}(\textbf{h}_p,\textbf{t}_p),
\end{align*}
where $\lambda\in\mathbb{R}$ is a parameter that learned by the model.
The corresponding score function is
\begin{align*}
    f_r(\textbf{h},\textbf{t})=d_r(\textbf{h},\textbf{t})=-d_{r,m}(\textbf{h},\textbf{t})-\lambda d_{r,p}(\textbf{h},\textbf{t}).
\end{align*}

When two entities have the same moduli, then the modulus part $d_{r,m}(\textbf{h}_m,\textbf{t}_m)=0$. However, the phase part $d_{r,p}(\textbf{h}_p,\textbf{t}_p)$ can be very different. By combining the modulus part and the phase part, HAKE can model the entities in both the category (a) and the category (b). Therefore, HAKE can model semantic hierarchies of knowledge graphs.

When evaluating the models, we find that adding a \textbf{mixture bias} to $d_{r,m}(\textbf{h},\textbf{t})$ can help to improve the performance of HAKE. The modified $d_{r,m}(\textbf{h},\textbf{t})$ is given by:
\begin{align*}\label{eqn:bias}
    d'_{r,m}(\textbf{h},\textbf{t})=\|\textbf{h}_m\circ \textbf{r}_m+(\textbf{h}_m+\textbf{t}_m)\circ \textbf{r}'_m-\textbf{t}_m\|_2,
\end{align*}
where $-\textbf{r}_m<\textbf{r}'_m<1$ is a vector that have the same dimension with $\textbf{r}_m$.
Indeed, the above distance function is equivalent to 
\begin{align*}
    d'_{r,m}(\textbf{h},\textbf{t})=\|\textbf{h}_m\circ ((\textbf{r}_m+\textbf{r}_m')/(1-\textbf{r}_m'))-\textbf{t}_m\|_2,
\end{align*}
where $/$ denotes the element-wise division operation.
If we let $\textbf{r}_m\leftarrow(\textbf{r}_m+\textbf{r}_m')/(1-\textbf{r}_m')$, then the modified distance function is exactly the same as the original one when compare the distances of different entity pairs. For notation convenience, we still use $d_{r,m}(\textbf{h},\textbf{t})=\|\textbf{h}_m\circ \textbf{r}_m-\textbf{t}_m\|_2$ to represent the modulus part. We will conduct ablation studies on the bias in the experiment section.

\subsection{Loss Function}
To train the model, we use the negative sampling loss functions with self-adversarial training \citep{rotate}:
\begin{align*}
    L=&-\log\sigma(\gamma-d_r(\textbf{h},\textbf{t}))\\&-\sum_{i=1}^np(h'_i,r,t'_i)\log\sigma(d_r(\textbf{h}'_i,\textbf{t}'_i)-\gamma),
\end{align*}
where $\gamma$ is a fixed margin, $\sigma$ is the sigmoid function, and $(h'_i,r,t'_i)$ is the $i$th negative triple. Moreover,
\begin{align*}
    p(h'_j,r,t'_j|\{(h_i,r_i,t_i)\})=\frac{\exp \alpha f_r(\textbf{h}'_j, \textbf{t}'_j)}{\sum_i \exp \alpha f_r(\textbf{h}'_i, \textbf{t}'_i)}
\end{align*}
is the probability distribution of sampling negative triples, where $\alpha$ is the temperature of sampling.

\section{Experiments and Analysis}
This section is organized as follows. First, we introduce the experimental settings in detail. Then, we show the effectiveness of our proposed model on three benchmark datasets. Finally, we analyze the embeddings generated by HAKE, and show the results of ablation studies. The code of HAKE is available on GitHub at \url{https://github.com/MIRALab-USTC/KGE-HAKE}.

\subsection{Experimental Settings}
We evaluate our proposed models on three commonly used knowledge graph datasets---WN18RR \citep{wn18rr}, FB15k-237 \citep{conve}, and YAGO3-10 \citep{yago3}. Details of these datasets are summarized in Table \ref{table:datasets}.

WN18RR, FB15k-237, and YAGO3-10 are subsets of WN18 \citep{transe}, FB15k \citep{transe}, and YAGO3 \citep{yago3}, respectively.
As pointed out by \citet{wn18rr} and \citet{conve}, WN18 and FB15k suffer from the test set leakage problem. One can attain the state-of-the-art results even using a simple rule based model. Therefore, we use WN18RR and FB15k-237 as the benchmark datasets. 

\noindent\textbf{Evaluation Protocol} Following \citet{transe}, for each triple $(h,r,t)$ in the test dataset, we replace either the head entity $h$ or the tail entity $t$ with each candidate entity to create a set of candidate triples. We then rank the candidate triples in descending order by their scores. It is worth noting that we use the ``Filtered'' setting as in \citet{transe}, which does not take any existing valid triples into accounts at ranking. We choose Mean Reciprocal Rank (MRR) and Hits at N (H@N) as the evaluation metrics. Higher MRR or H@N indicate better performance. 

\noindent\textbf{Training Protocol} We use Adam \citep{adam} as the optimizer, and use grid search to find the best hyperparameters based on the performance on the validation datasets. To make the model easier to train, we add an additional coefficient to the distance function, i.e., $d_{r}(\textbf{h},\textbf{t})=\lambda_1d_{r,m}(\textbf{h}_m,\textbf{t}_m)+\lambda_2 d_{r,p}(\textbf{h}_p,\textbf{t}_p)$, where $\lambda_1,\lambda_2\in \mathbb{R}$. 

\noindent\textbf{Baseline Model} One may argue that the phase part is unnecessary, as we can distinguish entities in the category (b) by allowing $[\textbf{r}]_i$ to be negative. We propose a model---ModE---that uses only the modulus part but allow $[\textbf{r}]_i<0$. Specifically, the distance function of ModE is
\begin{align*}
    d_r(\textbf{h},\textbf{t})=\|\textbf{h}\circ \textbf{r}-\textbf{t}\|_2, \text{where } \textbf{h},\textbf{r},\textbf{t}\in\mathbb{R}^k.
\end{align*}

\subsection{Main Results}
In this part, we show the performance of our proposed models---HAKE and ModE---against existing state-of-the-art methods, including TransE \citep{transe}, DistMult \citep{distmult}, ComplEx \citep{complex}, ConvE \citep{conve}, and RotatE \citep{rotate}. 

Table \ref{table:main_results} shows the performance of HAKE, ModE, and several previous models. Our baseline model ModE shares similar simplicity with TransE, but significantly outperforms it on all datasets. Surprisingly, ModE even outperforms more complex models such as DistMult, ConvE and Complex on all datasets, and beats the state-of-the-art model---RotatE---on FB15k-237 and YAGO3-10 datasets, which demonstrates the great power of modulus information. Table \ref{table:main_results} also shows that our HAKE significantly outperforms existing state-of-the-art methods on all datasets. 

\begin{table}[ht]
    \centering
    \caption{Statistics of datasets. The symbols \#E and \#R denote the number of entities and relations, respectively. \#TR, \#VA, and \#TE denote the size of train set, validation set, and test set, respectively.}
    \resizebox{.95\columnwidth}{!}{
    \begin{tabular}{l *{5}{c}}
        \toprule
        Dataset & \#E & \#R & \#TR & \#VA & \#TE \\
        \midrule
        WN18RR & 40,493 & 11 & 86,835 & 3,034 & 3,134 \\
        FB15k-237 & 14,541 & 237 & 272,115 & 17,535 & 20,466 \\
        YAGO3-10 & 123,182 & 37 & 1,079,040 & 5,000 & 5,000 \\
        \bottomrule
    \end{tabular}
    }
    \label{table:datasets}
\end{table}

WN18RR dataset consists of two kinds of relations: the symmetric relations such as $\_similar\_to$, which link entities in the category (b); other relations such as $\_hypernym$ and $\_member\_meronym$, which link entities in the category (a). Actually, RotatE can model entities in the category (b) very well \citep{rotate}. However, HAKE gains a 0.021 higher MRR, a 2.4\% higher H@1, and a 2.4\% higher H@3 against RotatE, respectively. The superior performance of HAKE compared with RotatE implies that our proposed model can better model different levels in the hierarchy.

FB15k-237 dataset has more complex relation types and fewer entities, compared with WN18RR and YAGO3-10. Although there are relations that reflect hierarchy in FB15k-237, there are also lots of relations, such as ``/location/location/time\_zones'' and ``/film/film/prequel'', that do not lead to hierarchy. The characteristic of this dataset accounts for why our proposed models doesn't outperform the previous state-of-the-art as much as that of WN18RR and YAGO3-10 datasets. However, the results also show that our models can gain better performance so long as there exists semantic hierarchies in knowledge graphs. As almost all knowledge graphs have such hierarchy structures, our model is widely applicable.  

YAGO3-10 datasets contains entities with high relation-specific indegree \citep{conve}. For example, the link prediction task $(?, hasGender, male)$ has over $1000$ true answers, which makes the task challenging. Fortunately, we can regard ``male'' as an entity at higher level of the hierarchy and the predicted head entities as entities at lower level. In this way, YAGO3-10 is a dataset that clearly has semantic hierarchy property, and we can expect that our proposed models is capable of working well on this dataset. Table \ref{table:main_results} validates our expectation. Both ModE and HAKE significantly outperform the previous state-of-the-art. Notably, HAKE gains a 0.050 higher MRR, 6.0\% higher H@1 and 4.6\% higher H@3 than RotatE, respectively.

\begin{table*}[ht]
    \caption{Evaluation results on WN18RR, FB15k-237 and YAGO3-10 datasets. Results of TransE and RotatE are taken from \citet{convkb} and \citet{rotate}, respectively. Other results are taken from \citet{conve}.}
    \centering
    \begin{tabular}{l  c c c c  c c c c  c c c c }
        \toprule
          &\multicolumn{4}{c}{\textbf{WN18RR}}&  \multicolumn{4}{c}{\textbf{FB15k-237}} & \multicolumn{4}{c}{\textbf{YAGO3-10}}\\
         \cmidrule(lr){2-5}
         \cmidrule(lr){6-9}
         \cmidrule(lr){10-13}
         & MRR & H@1 & H@3 & H@10 & MRR & H@1 & H@3 & H@10 & MRR & H@1 & H@3 & H@10 \\
        \midrule
        TransE & .226 &   -  &   -  & .501 & .294 &   -  &   -  & .465 & - & - & - & -\\
        DistMult & .43 & .39 & .44 & .49 & .241 & .155 & .263 & .419 & .34 & .24 & .38 & .54 \\
        ConvE  & .43  & .40 & .44 & .52 & .325 & .237 & .356 & .501 & .44  & .35  & .49  & .62\\
        ComplEx & .44  & .41  & .46  & .51  & .247 & .158 & .275 & .428 & .36 & .26 & .40 & .55\\
        RotatE & .476 & .428 & .492 & .571 & .338 & .241 & .375 & .533 & .495 & .402 & .550 & .670\\
        \midrule
        ModE  & .472 & .427 & .486 & .564 & .341 & .244 & .380 & .534 & .510 & .421 & .562 & .660\\
        HAKE   & \textbf{.497} & \textbf{.452} & \textbf{.516} & \textbf{.582} & \textbf{.346} & \textbf{.250} & \textbf{.381} & \textbf{.542} & \textbf{.545} & \textbf{.462} & \textbf{.596} & \textbf{.694}\\
        \bottomrule
    \end{tabular}
    \label{table:main_results}
\end{table*}

\begin{figure}[!ht]
    \centering 
    \includegraphics{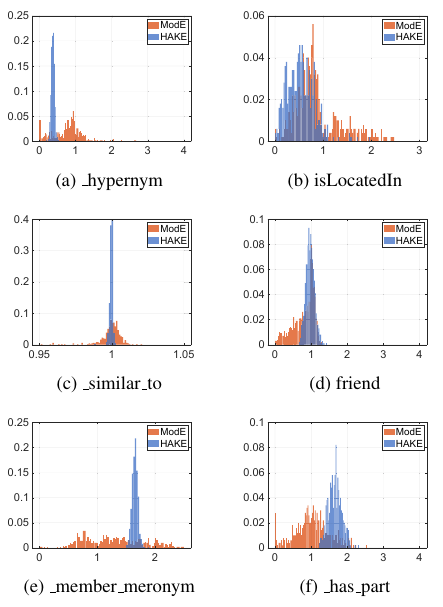}
\caption{Distribution histograms of moduli of some relations. The relations are drawn from WN18RR, FB15k-237 and YAGO3-10 dataset. The relation in (d) is \textit{/celebrities/celebrity/celebrity\_friends/celebrities/friendship/friend}. Let \textit{friend} denote the relation for simplicity.}
\label{fig:wn18rr_histogram}
\end{figure}

\subsection{Analysis on Relation Embeddings}
In this part, we first show that HAKE can effectively model the hierarchy structures by analyzing the moduli of relation embeddings. Then, we show that the phase part of HAKE can help us to distinguish entities at the same level of the hierarchy by analyzing the phases of relation embeddings.

\begin{figure}[!ht]
    \centering 
    \includegraphics{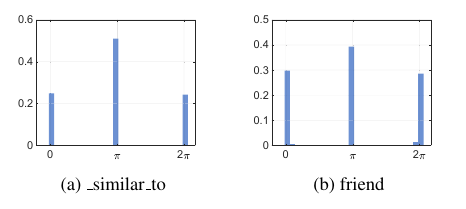}
\caption{Distribution histograms of phases of two relations that reflect the same hierarchy. The relations in Figure (a) and (b) are drawn from WN18RR and FB15k-237, respectively.}
\label{fig:phase_hist}
\end{figure}

In Figure \ref{fig:wn18rr_histogram}, we plot the distribution histograms of moduli of six relations. These relations are drawn from WN18RR, FB15k-237, and YAGO3-10. Specifically, the relations in Figures \ref{fig:wn18rr_histogram}a, \ref{fig:wn18rr_histogram}c, \ref{fig:wn18rr_histogram}e and \ref{fig:wn18rr_histogram}f  are drawn from WN18RR. The relation in Figure \ref{fig:wn18rr_histogram}d is drawn from FB15k-237. The relation in Figure \ref{fig:wn18rr_histogram}b is drawn from YAGO3-10. We divide the relations in Figure \ref{fig:wn18rr_histogram} into three groups.
\begin{enumerate}[(A)]
    \item Relations in Figures \ref{fig:wn18rr_histogram}c and \ref{fig:wn18rr_histogram}d
    connect the entities at the same level of the semantic hierarchy;
    \item Relations in Figures \ref{fig:wn18rr_histogram}a and \ref{fig:wn18rr_histogram}b
    represent that tail entities are at higher levels than head entities of the hierarchy;
    \item Relations in Figures \ref{fig:wn18rr_histogram}e and \ref{fig:wn18rr_histogram}f
    represent that tail entities are at lower levels than head entities of the hierarchy.
\end{enumerate}

\begin{table*}[ht]
    \caption{Ablation results on WN18RR, FB15k-237 and YAGO3-10 datasets. The symbols \textbf{m}, \textbf{p}, and \textbf{b} represent the modulus part, the phase part, and the mixture bias term, respectively.}
    \centering
    \begin{tabular}{ccc c c c c  c c c c  c c c c }
        \toprule
          &&&\multicolumn{4}{c}{\textbf{WN18RR}}&  \multicolumn{4}{c}{\textbf{FB15k-237}} & \multicolumn{4}{c}{\textbf{YAGO3-10}}\\
         \cmidrule(lr){4-7}
         \cmidrule(lr){8-11}
         \cmidrule(lr){12-15}
         \textbf{m} & \textbf{p}& \textbf{b}& MRR & H@1 & H@3 & H@10 & MRR & H@1 & H@3 & H@10 & MRR & H@1 & H@3 & H@10 \\
        \midrule
        \checkmark&&& .240 & .047 & .404 & .527 & .258 & .121 & .333 & .508 & .476 & .374 & .541 & .658\\
        &\checkmark&& .465 & .423 & .480 & .550 & .324 & .226 & .361 & .519 & .480 & .383 & .532 & .664 \\
        \midrule
        \checkmark&\checkmark& & .496 & .449 & \textbf{.517} & \textbf{.584} & .336 & .239 & .373 & .533 & .522 & .429 & .581 & .693\\
         \midrule
           \checkmark &\checkmark&\checkmark& \textbf{.497} & \textbf{.452} & .516 & .582 & \textbf{.346} & \textbf{.250} & \textbf{.381} & \textbf{.542} & \textbf{.545} & \textbf{.462} & \textbf{.596} & \textbf{.694} \\
        \bottomrule
    \end{tabular}
    \label{table:ablation_bias}
\end{table*}
As described in the model description section, we expect entities at higher levels of the hierarchy to have small moduli. The experiments validate our expectation. For both ModE and HAKE, most entries of the relations in the group (A) take values around one, which leads to that the head entities and tail entities have approximately the same moduli. In the group (B), most entries of the relations take values less than one, which results in that the head entities have smaller moduli than the tail entities. The cases in the group (C) are contrary to that in the group (B). These results show that our model can capture the semantic hierarchies in knowledge graphs.
Moreover, compared with ModE, the relation embeddings' moduli of HAKE have lower variances, which shows that HAKE can model hierarchies more clearly. 

As mentioned above, relations in the group (A) reflect the same semantic hierarchy, and are expected to have the moduli of about one. Obviously, it is hard to distinguish entities linked by these relations only using the modulus part. In Figure \ref{fig:phase_hist}, we plot the phases of the relations in the group (A). The results show that the entities at the same level of the hierarchy can be distinguished by their phases, as many phases have the values of $\pi$.

\begin{table}[ht]
    \centering
    \caption{Comparison results with TKRL models \citep{hType} on FB15k dataset. RHE, WHE, RHE+STC, and WHE+STC are four versions of TKRL model , of which the results are taken from the original paper.}
    \resizebox{.95\columnwidth}{!}{
    \begin{tabular}{l *{5}{c}}
    \toprule
    &HAKE &RHE &WHE &RHE+STC &WHE+STC\\ 
    \midrule
    H@10 &\textbf{.884} &.694 &.696 &.731 &.734 \\
    \bottomrule
    \end{tabular}
    }
    \label{table:cmp_tkrl}
\end{table}

\subsection{Analysis on Entity Embeddings}
In this part, to further show that HAKE can capture the semantic hierarchies between entities, we visualize the embeddings of several entity pairs.

We plot the entity embeddings of two models: the previous state-of-the-art RotatE and our proposed HAKE. RotatE regards each entity as a group of complex numbers. As a complex number can be seen as a point on a 2D plane, we can plot the entity embeddings on a 2D plane. As for HAKE, we have mentioned that it maps entities into the polar coordinate system. Therefore, we can also plot the entity embeddings generated by HAKE on a 2D plane based on their polar coordinates. For a fair comparison, we set $k=500$. That is, each plot contains $500$ points, and the actual dimension of entity embeddings is $1000$.  Note that we use the logarithmic scale to better display the differences between entity embeddings. As all the moduli have values less than one, after applying the logarithm operation, the larger radii in the figures will actually represent smaller modulus.

Figure \ref{fig:scatter_modulus} shows the visualization results of three triples from the WN18RR dataset. Compared with the tail entities, the head entities in Figures \ref{fig:scatter_modulus}a, \ref{fig:scatter_modulus}b, and \ref{fig:scatter_modulus}c are at lower levels, similar levels, higher levels in the semantic hierarchy, respectively. We can see that there exist clear concentric circles in the visualization results of HAKE, which demonstrates that HAKE can effectively model the semantic hierarchies. However, in RotatE, the entity embeddings in all three subfigures are mixed, making it hard to distinguish entities at different levels in the hierarchy.


\begin{figure}[!ht]
  \centering 
 \begin{subfigure}[b]{0.4\textwidth}
  \centering
  \includegraphics[width=93pt]{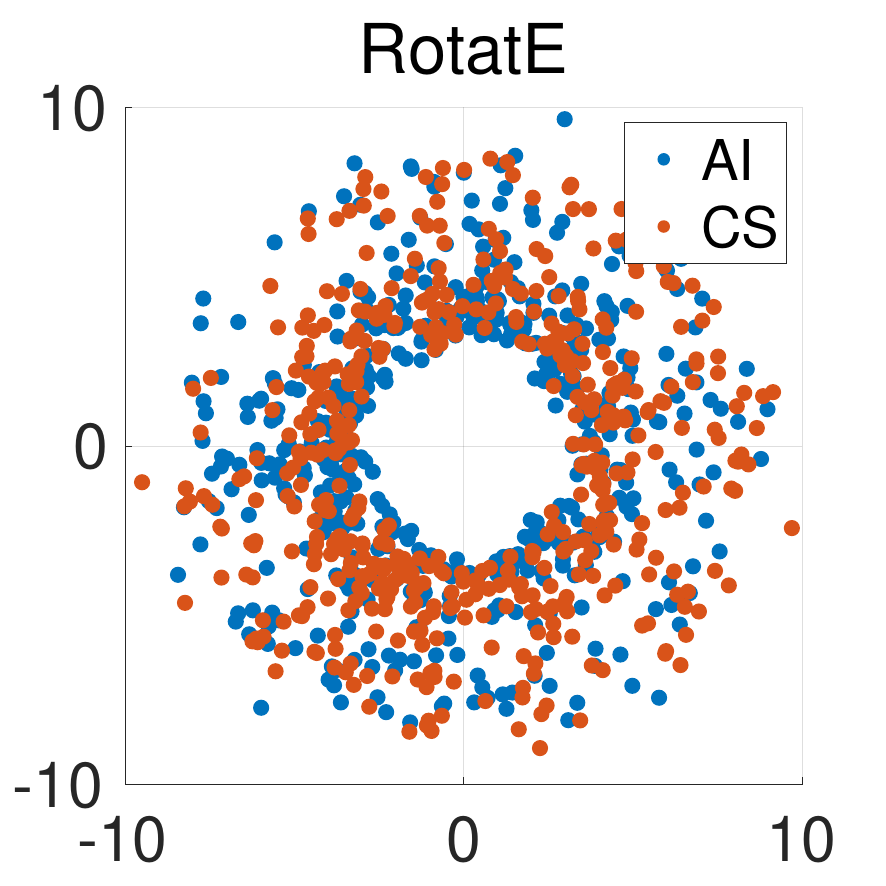}\hspace{4mm}
  \includegraphics[width=93pt]{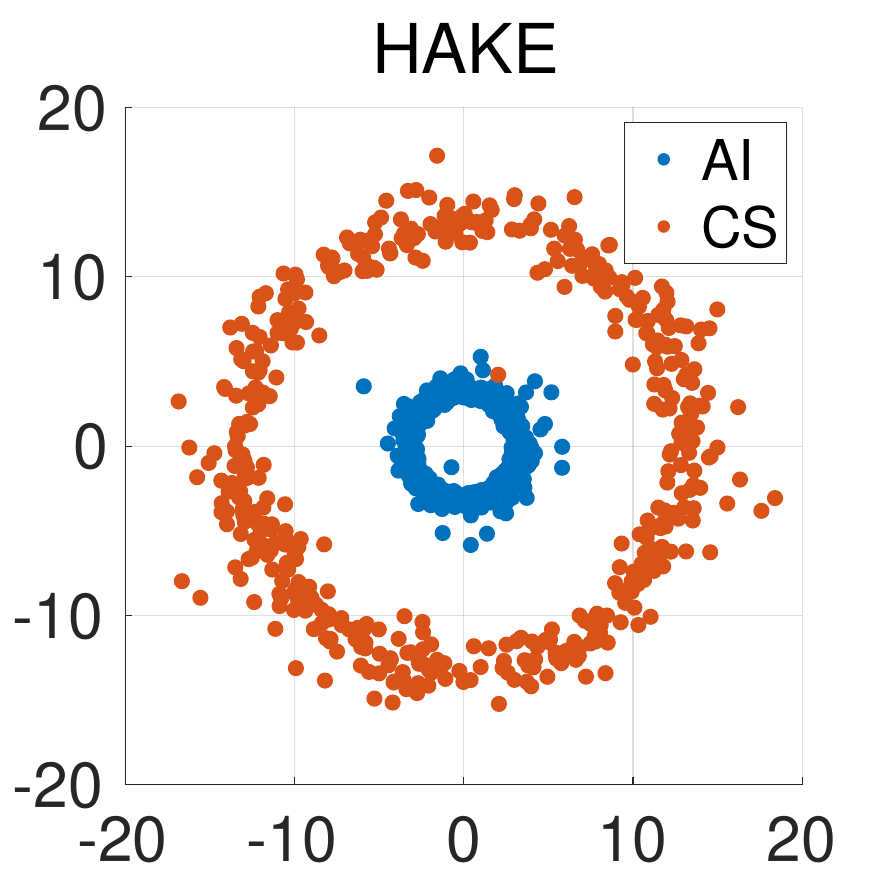}
  \caption{\textit{(AI, \_hypernym, CS)}}
  \label{fig:scatter_modulus_sub1}
\end{subfigure} 

\vspace{2mm}
\begin{subfigure}[b]{0.4\textwidth}
  \centering
  \includegraphics[width=93pt]{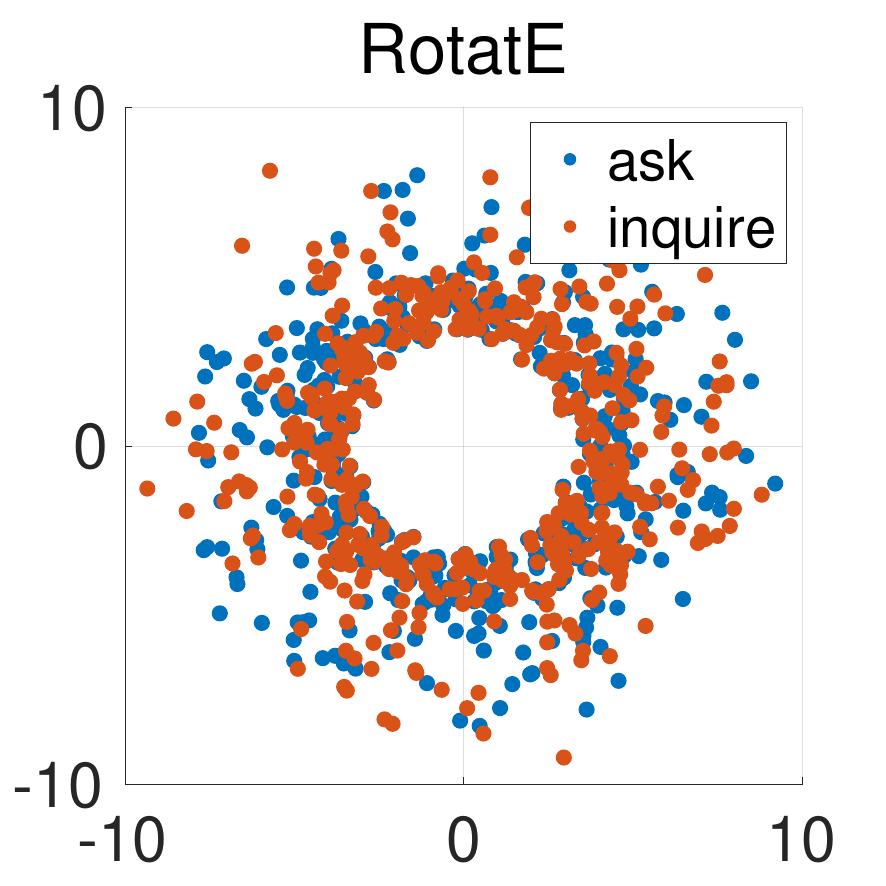}\hspace{4mm}
  \includegraphics[width=93pt]{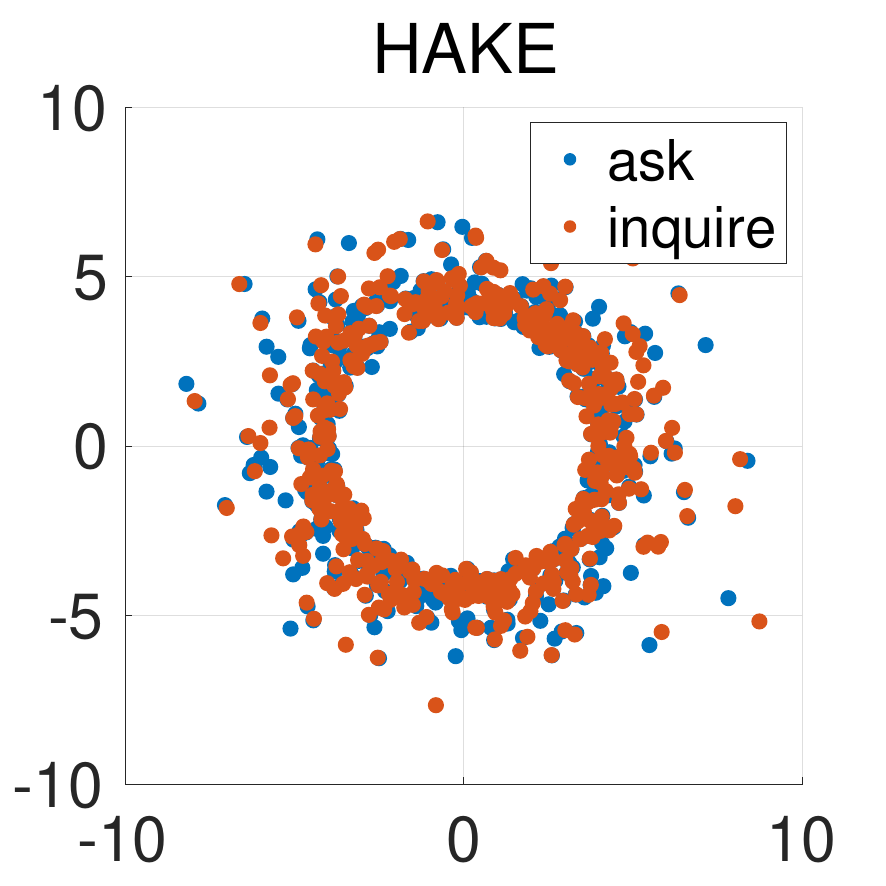}
  \caption{\textit{(ask, \_verb\_group, inquire)}}
  \label{fig:scatter_modulus_sub2}
\end{subfigure} 

\vspace{2mm}
\begin{subfigure}[b]{0.4\textwidth}
  \centering
  \includegraphics[width=93pt]{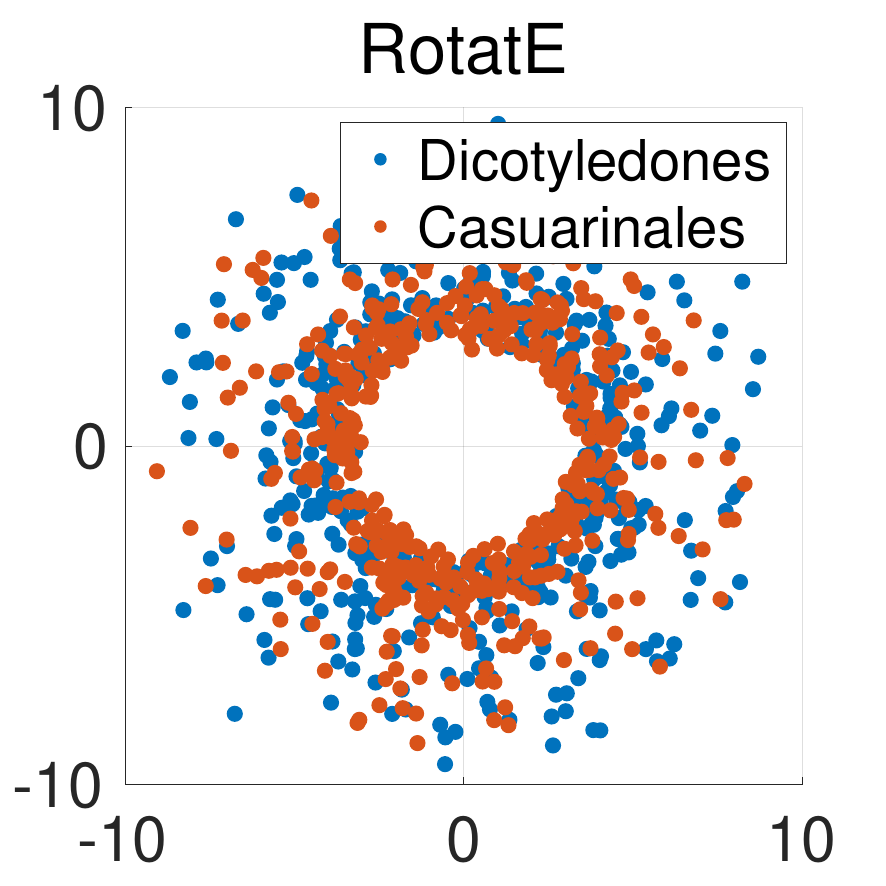}\hspace{4mm}
  \includegraphics[width=93pt]{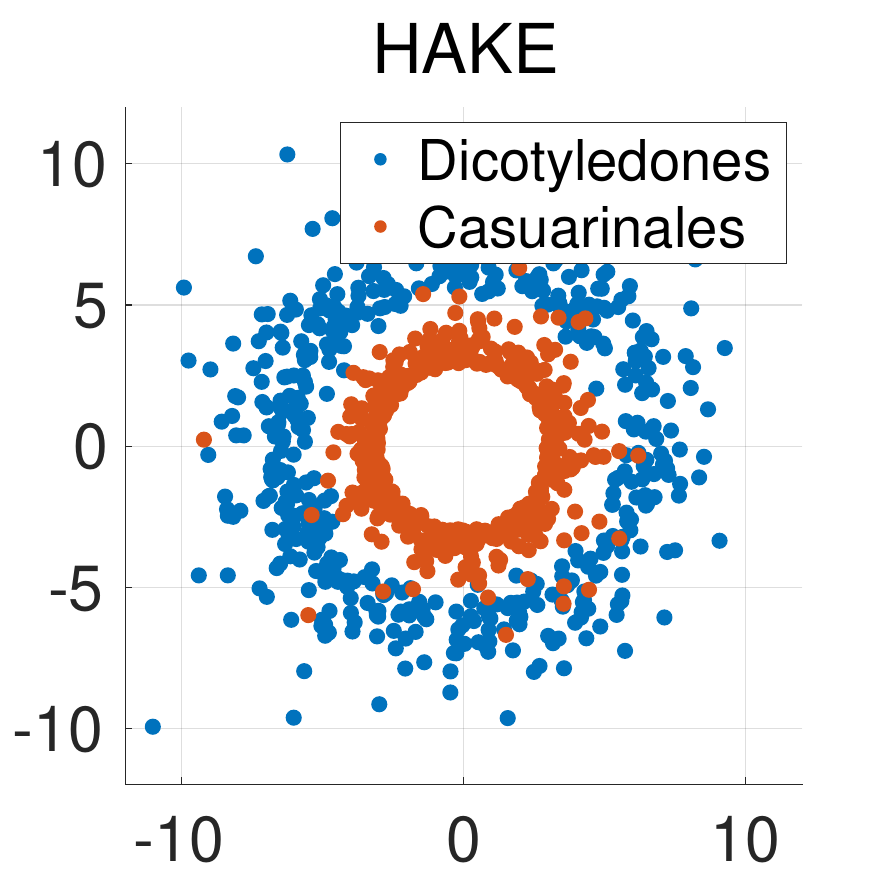}
  \caption{\textit{(Dicotyledones, \_member\_meronym, Casuarinales)}}
  \label{fig:scatter_modulus_sub3}
\end{subfigure} 
\caption{Visualization of the embeddings of several entity pairs from WN18RR dataset.
}
\label{fig:scatter_modulus}
\end{figure}

\subsection{Ablation Studies}
In this part, we conduct ablation studies on the modulus part and the phase part of HAKE, as well as the mixture bias item.  Table \ref{table:ablation_bias} shows the results on three benchmark datasets.

We can see that the bias can improve the performance of HAKE on nearly all metrics. Specifically, the bias improves the H@1 score of $4.7\%$ on YAGO3-10 dataset, which illustrates the effectiveness of the bias.

We also observe that the modulus part of HAKE does not perform well on all datasets, due to its inability to distinguish the entities at the same level of the hierarchy. When only using the phase part, HAKE degenerates to the pRotatE model \citep{rotate}. It performs better than the modulus part, because it can well model entities at the same level of the hierarchy. However, our HAKE model significantly outperforms the modulus part and the phase part on all datasets, which demonstrates the importance to combine the two parts for modeling semantic hierarchies in knowledge graphs.

\subsection{Comparison with Other Related Work}
We compare our models with TKRL models \citep{hType}, which also aim to model the hierarchy structures. For the difference between HAKE and TKRL, please refer to the Related Work section.  Table \ref{table:cmp_tkrl} shows the H@10 scores of HAKE and TKRLs on FB15k dataset. The best performance of TKRL is .734 obtained by the WHE+STC version, while the H@10 score of our HAKE model is .884. The results show that HAKE significantly outperforms TKRL, though it does not require additional information.

\section{Conclusion}
To model the semantic hierarchies in knowledge graphs, we propose a novel hierarchy-aware knowledge graph embedding model---HAKE---which maps entities into the polar coordinate system. Experiments show that our proposed HAKE significantly outperforms several existing state-of-the-art methods on benchmark datasets for the link prediction task. A further investigation shows that HAKE is capable of modeling entities at both different levels and the same levels in the semantic hierarchies.

\section*{Acknowledgments}
This work was supported in part by National Science Foundations of China grants 61822604, U19B2026, 61836006, and 62021001, and the Fundamental Research Funds for the Central Universities grant WK3490000004.

\bibliographystyle{aaai}
\bibliography{3019.bibfile}

\newpage
\section*{Appendix}
In this appendix, we will provide analysis on relation patterns, negative entity embeddings, and moduli of entity embeddings. Then, we will give more visualization results on semantic hierarchies.

\section*{A.\,\,\, Analysis on Relation Patterns}
In this section, we prove that our HAKE model can infer the (anti)symmetry, inversion and composition relation patterns. Detailed propositions and their proofs are as follows.

\begin{prop}
    HAKE can infer the (anti)symmetry pattern.
\end{prop}

\begin{proof}
    If $r(x, y)$ and $r(y, x)$ hold, we have
    \begin{align*}
        \begin{cases}
            \textbf{x}_m\circ \textbf{r}_m=\textbf{y}_m, \\ \textbf{y}_m\circ \textbf{r}_m=\textbf{x}_m, \\ (\textbf{x}_p+\textbf{r}_p) \mod 2\pi=\textbf{y}_p, \\ (\textbf{y}_p+\textbf{r}_p) \mod 2\pi=\textbf{x}_p.
        \end{cases}
    \end{align*}
    Then we have
    \begin{gather*}
        \textbf{x}_m \circ \textbf{r}_m \circ \textbf{r}_m =\textbf{x}_m \Rightarrow \textbf{r}_m \circ \textbf{r}_m = \textbf{1}, \\
        (\textbf{r}_p + \textbf{r}_p) \mod 2\pi = 0.
    \end{gather*}

    Otherwise, if $r(x, y)$ and $\neg r(y, x)$ hold, we have
    \begin{gather*}
        \textbf{r}_m \circ \textbf{r}_m \neq \textbf{1}, \\
        (\textbf{r}_p + \textbf{r}_p) \mod 2\pi \neq 0.
    \end{gather*}
\end{proof}

\begin{prop}
    HAKE can infer the inversion pattern.
\end{prop}

\begin{proof}
    If $r_1(x, y)$ and $r_2(y, x)$ hold, we have
    \begin{align*}
        \begin{cases}
        \textbf{x}_m\circ \textbf{r}_{1,m}=\textbf{y}_m, \\ \textbf{y}_m\circ \textbf{r}_{2,m}=\textbf{x}_m, \\ (\textbf{x}_p+\textbf{r}_{1,p}) \mod 2\pi=\textbf{y}_p, \\ (\textbf{y}_p+\textbf{r}_{2,p}) \mod 2\pi=\textbf{x}_p.
        \end{cases}
    \end{align*}
    
    Then, we have
    \begin{gather*}
    \textbf{x}_m \circ \textbf{r}_{1,m} \circ \textbf{r}_{2,m} =\textbf{x}_m \Rightarrow \textbf{r}_{1,m} = \textbf{r}_{2,m}^{-1}, \\
    (\textbf{r}_{1,p} + \textbf{r}_{2,p}) \mod 2\pi = 0.
    \end{gather*}
\end{proof}

\begin{figure}[!h]
  \centering 
  \begin{subfigure}[b]{0.2\textwidth}
  \includegraphics[width=105pt]{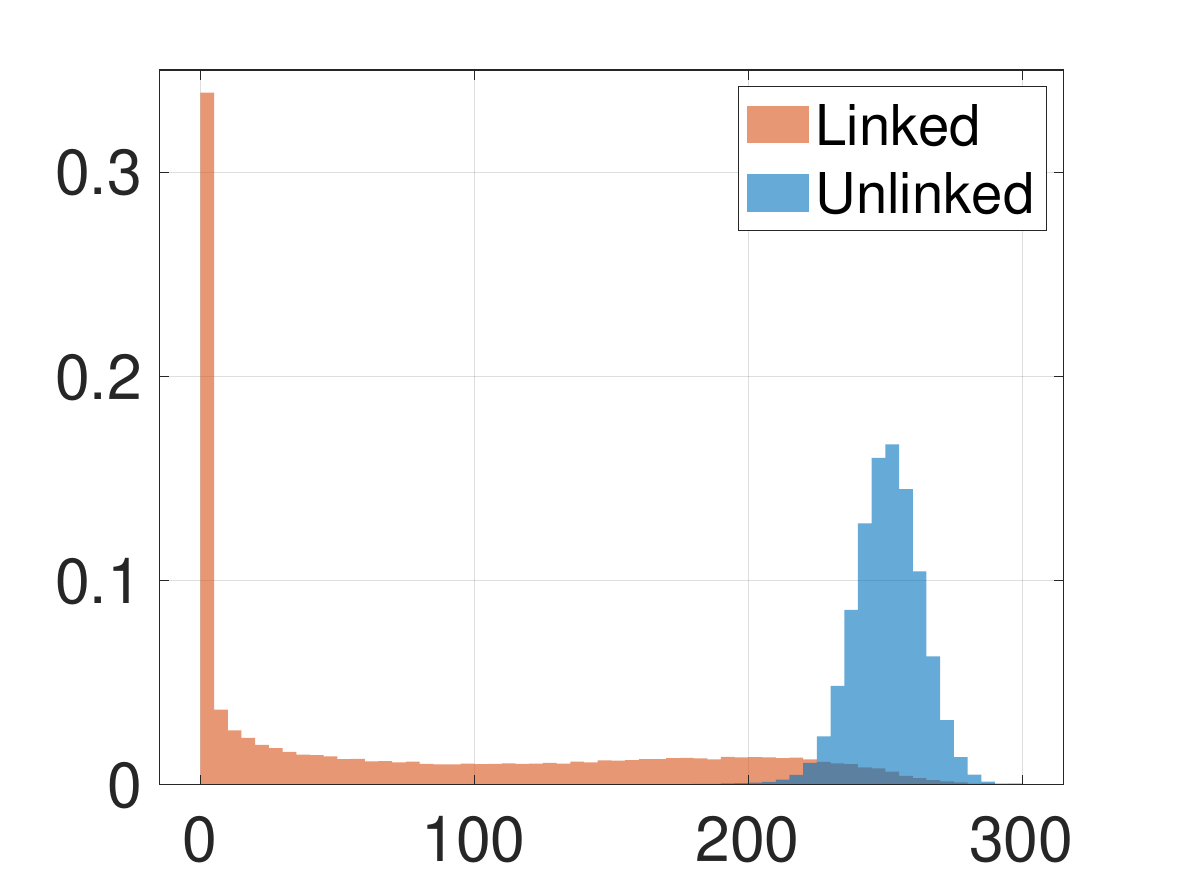}
  \caption{WN18RR}
  \end{subfigure}
  \begin{subfigure}[b]{0.2\textwidth}
  \includegraphics[width=105pt]{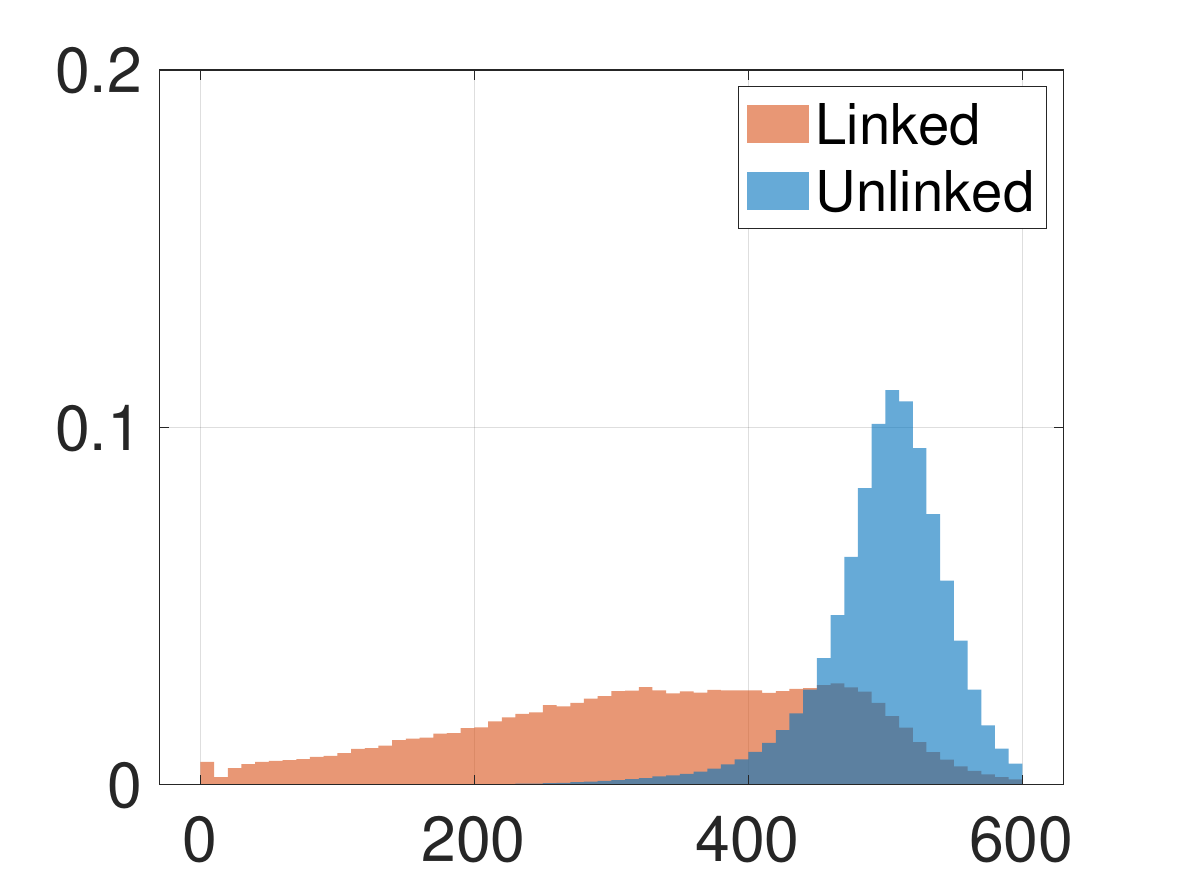}
  \caption{FB15k-237}
  \end{subfigure}
\caption{Illustration of the negative modulus of linked and unlinked entity pairs. For each pair of entities $h$ and $t$, if there is a link between $h$ and $t$, we label them as ``Linked''. Otherwise, we label them as ``Unlinked''. The x-axis represents the number of $i$ that $[\textbf{h}_m]_i$ and $[\textbf{t}_m]_i$ have different signs, the y-axis represents the frequency.}
\label{fig:neg_channel_illustration}
\end{figure}

\vspace{5mm}
\begin{figure}[!h]
  \centering 
  \begin{subfigure}[b]{0.2\textwidth}
  \includegraphics[width=105pt]{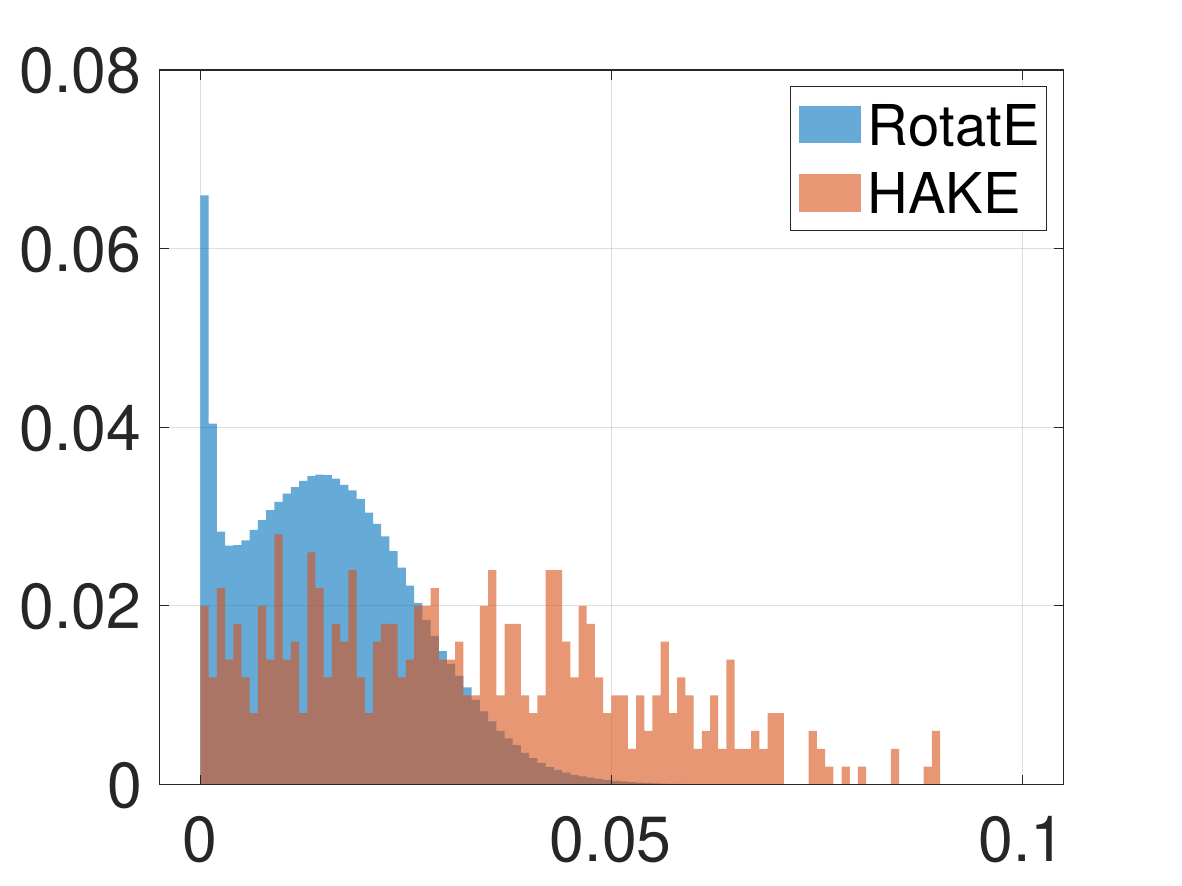}
  \caption{WN18RR}
  \end{subfigure}
  \begin{subfigure}[b]{0.2\textwidth}
  \includegraphics[width=105pt]{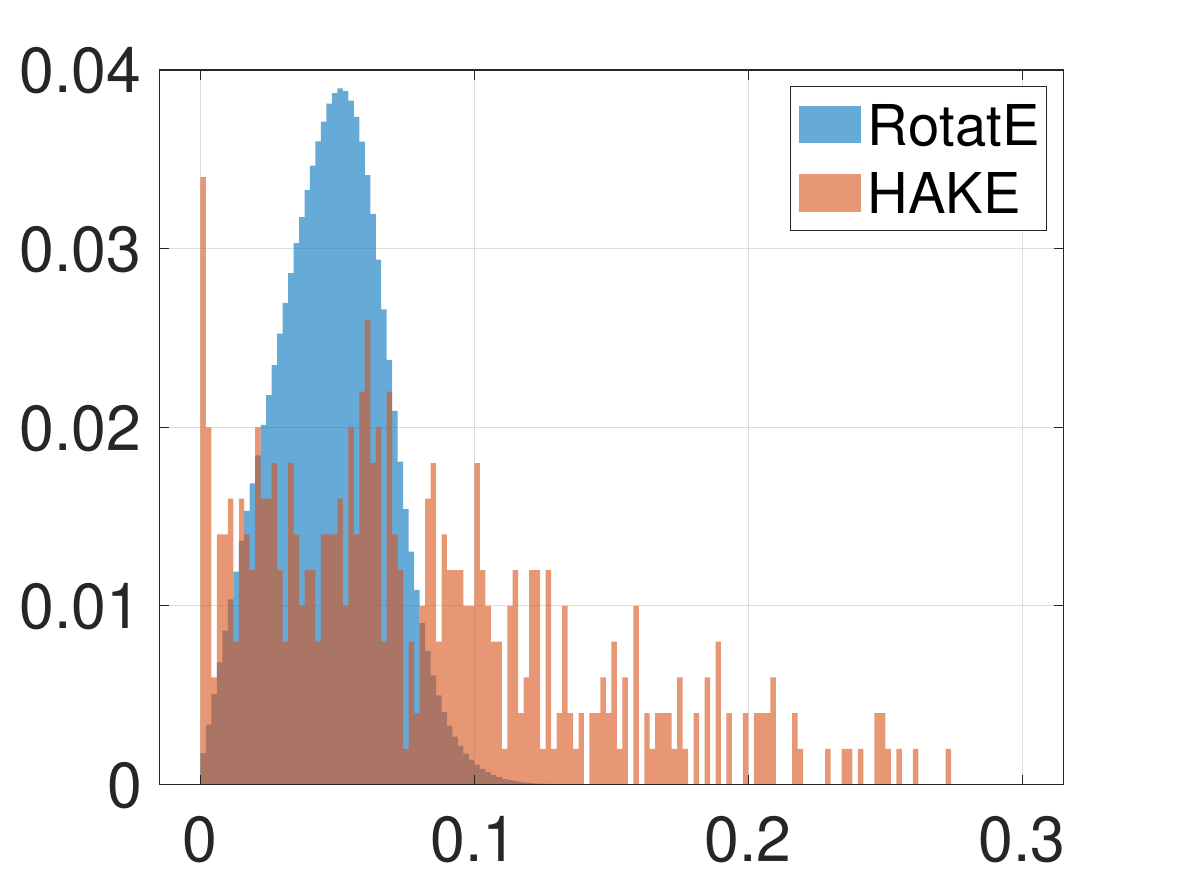}
  \caption{YAGO3-10}
  \end{subfigure}
\caption{Histograms of the modulus of entity embeddings. Compared with RotatE, the modulus of entity embeddings in HAKE are more dispersed, making it to have more potential to model the semantic hierarchies.}
\label{fig:total_modulus_hist}
\end{figure}

\begin{prop}
    HAKE can infer the composition pattern.
\end{prop}

\begin{proof}
    If $r_1(x, z)$, $r_2(x, y)$ and $r_3(y, z)$ hold, we have
    \begin{align*}
        \begin{cases}
        \textbf{x}_m\circ \textbf{r}_{1,m}=\textbf{z}_m, \\
        \textbf{x}_m\circ \textbf{r}_{2,m}=\textbf{y}_m, \\ 
        \textbf{y}_m\circ \textbf{r}_{3,m}=\textbf{z}_m, \\
        (\textbf{x}_p+\textbf{r}_{1,p}) \mod 2\pi=\textbf{z}_p, \\
        (\textbf{x}_p+\textbf{r}_{2,p}) \mod 2\pi=\textbf{y}_p, \\ 
        (\textbf{y}_p+\textbf{r}_{3,p}) \mod 2\pi=\textbf{z}_p.
        \end{cases}
    \end{align*}
     Then we have
    
    \begin{gather*}
        \textbf{x}_m \circ {\textbf{r}_2}_m \circ {\textbf{r}_3}_m \circ \textbf{z}_m = \textbf{x}_m \circ {\textbf{r}_1}_m \circ \textbf{z}_m \\
        \Rightarrow {\textbf{r}_1}_m = {\textbf{r}_2}_m \circ {\textbf{r}_3}_m, \\
        ({\textbf{r}_1}_p - {\textbf{r}_2}_p - {\textbf{r}_3}_p) \mod 2\pi = 0.
    \end{gather*}
    
\end{proof}



\section*{B.\,\,\, Analysis on Negative Entity Embeddings}
We denote the linked entity pairs as the set of entity pairs linked by some relation, and denote the unlinked entity pairs as the set of entity pairs that no triple contains in the train/valid/test dataset. It is worth noting that the unlinked paris may contain valid triples, as the knowledge graph is incomplete. For both the linked and the unlinked entity pairs, we count the embedding entries of two entities that have different signs. Figure \ref{fig:neg_channel_illustration} shows the result.

For the linked entity pairs, as we expected, most of the entries have the same sign. Due to the large amount of unlinked entity pairs, we randomly sample a part of them for plotting. For the unlinked entity pairs, around half of the entries have different signs, which is consistent with the random initialization. The results support our hypothesis that the negative signs of entity embeddings can help our model to distinguish positive and negative triples. 

\section*{C.\,\,\, Analysis on Moduli of Entity Embeddings}

Figure \ref{fig:total_modulus_hist} shows the modulus of entity embeddings. We can observe that RotatE encourages the modulus of embeddings to be the same, as the relations are modeled as rotations in a complex space. Compared with RotatE, the modulus of entity embeddings in HAKE are more dispersed, making it to have more potential to model the semantic hierarchies.

\section*{D.\,\,\, More Results on Semantic Hierarchies}

In this part, we visualize more triples from WN18RR. We plot the head and tail entities on 2D planes using the same method as that in the main text. The visualization results are in Figure \ref{fig:scatter_modulus2}, where the subcaptions demonstrate the corresponding triples. The figures show that, compared with RotatE, our HAKE model can better model the entities both in different hierarchies and in the same hierarchy. 

\vspace{5mm}
\begin{figure}[!ht]
  \centering 
 \begin{subfigure}[b]{0.4\textwidth}
  \centering
  \includegraphics[width=93pt]{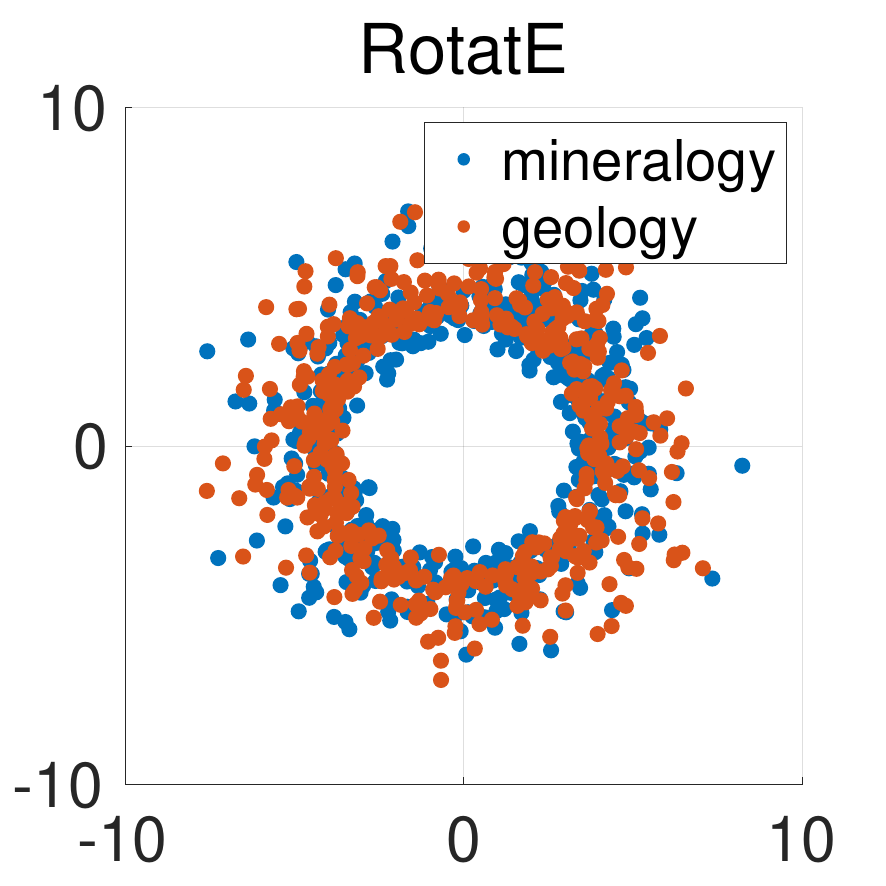}\hspace{4mm}
  \includegraphics[width=93pt]{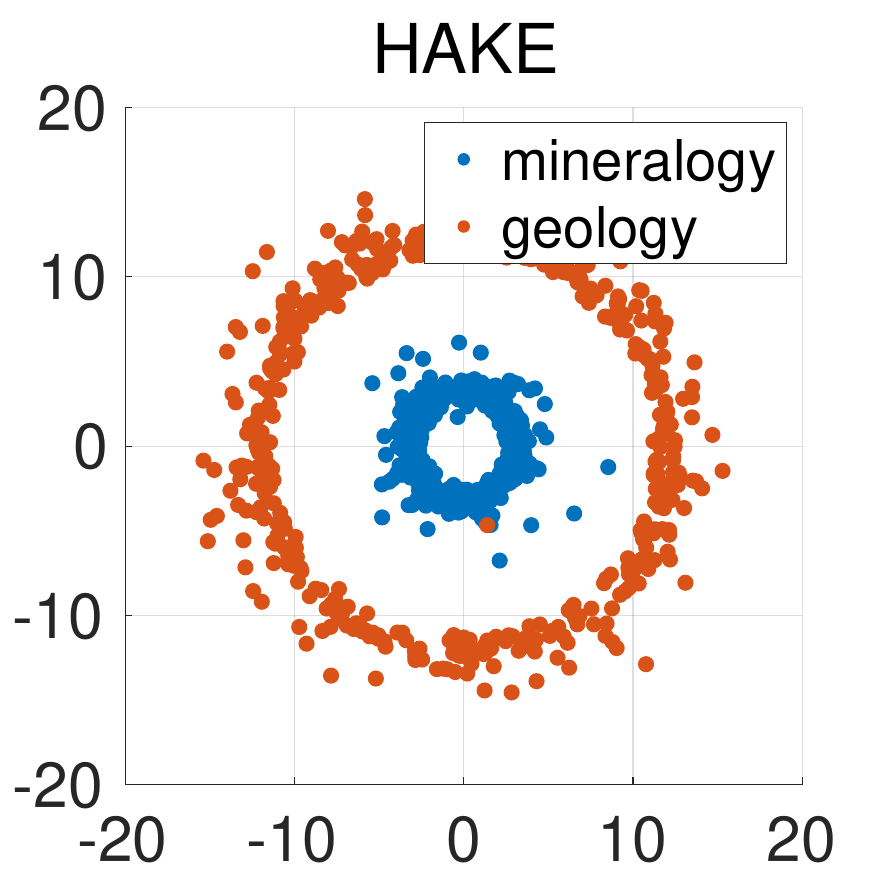}
  \caption{\textit{(mineralogy, \_hypernym, geology)}}
  \label{fig:scatter_modulus_sub1}
\end{subfigure} 

\vspace{2mm}
\begin{subfigure}[b]{0.4\textwidth}
  \centering
  \includegraphics[width=93pt]{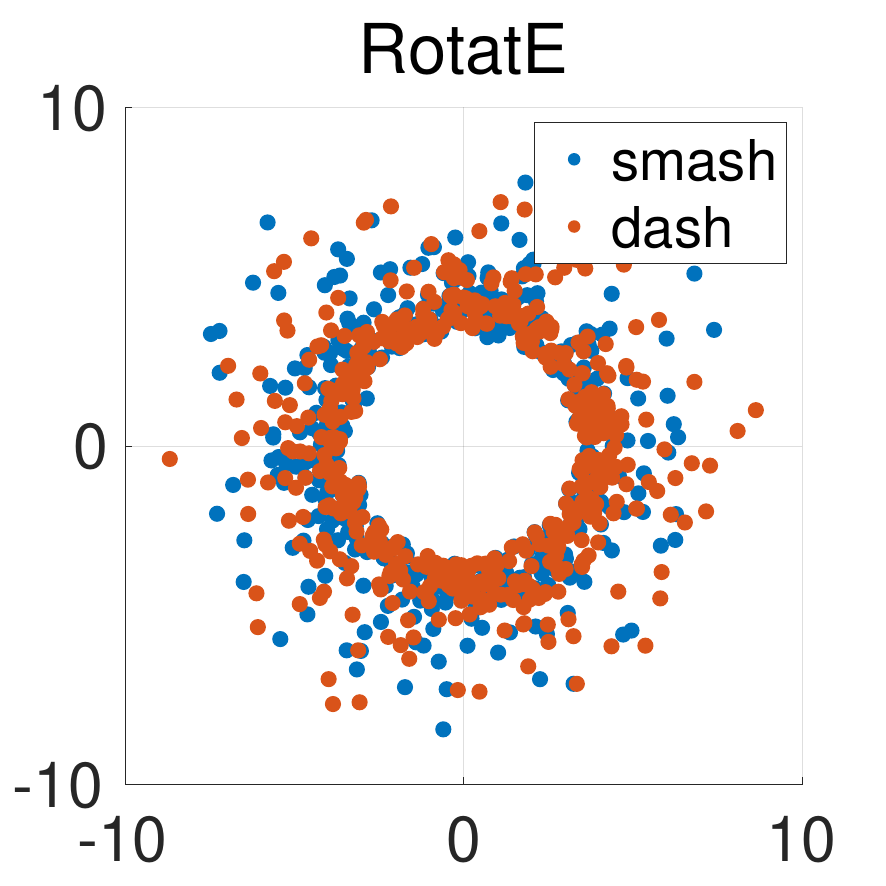}\hspace{4mm}
  \includegraphics[width=93pt]{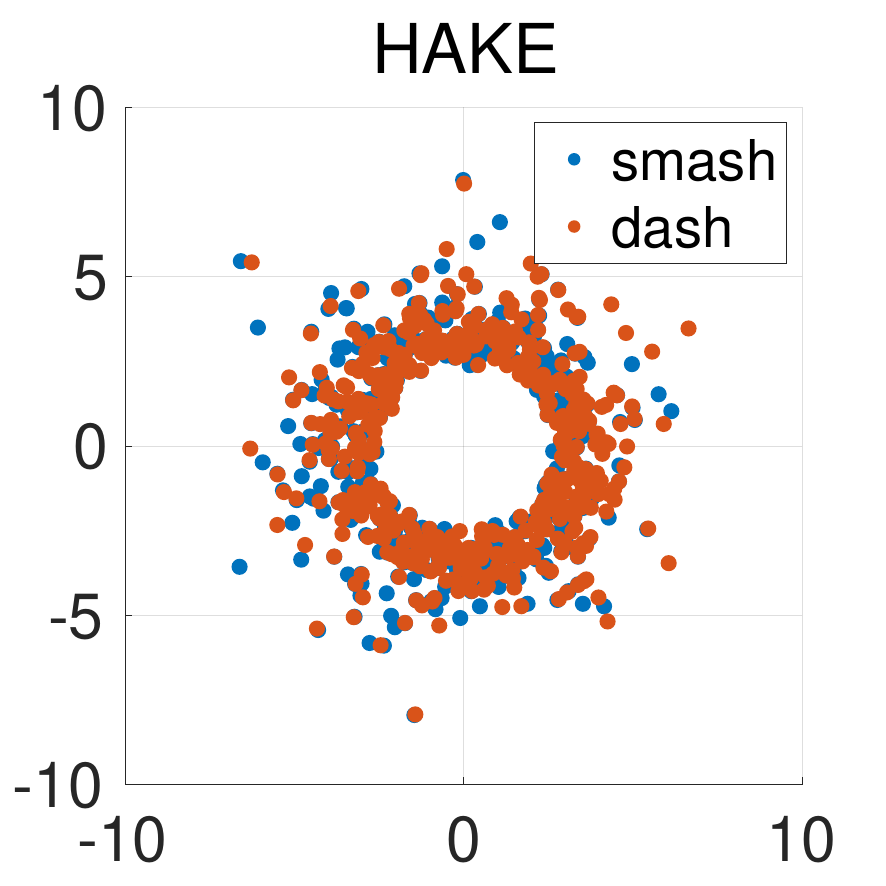}
  \caption{\textit{(smash, \_verb\_group, dash)}}
  \label{fig:scatter_modulus_sub2}
\end{subfigure} 

\vspace{2mm}
\begin{subfigure}[b]{0.4\textwidth}
  \centering
  \includegraphics[width=93pt]{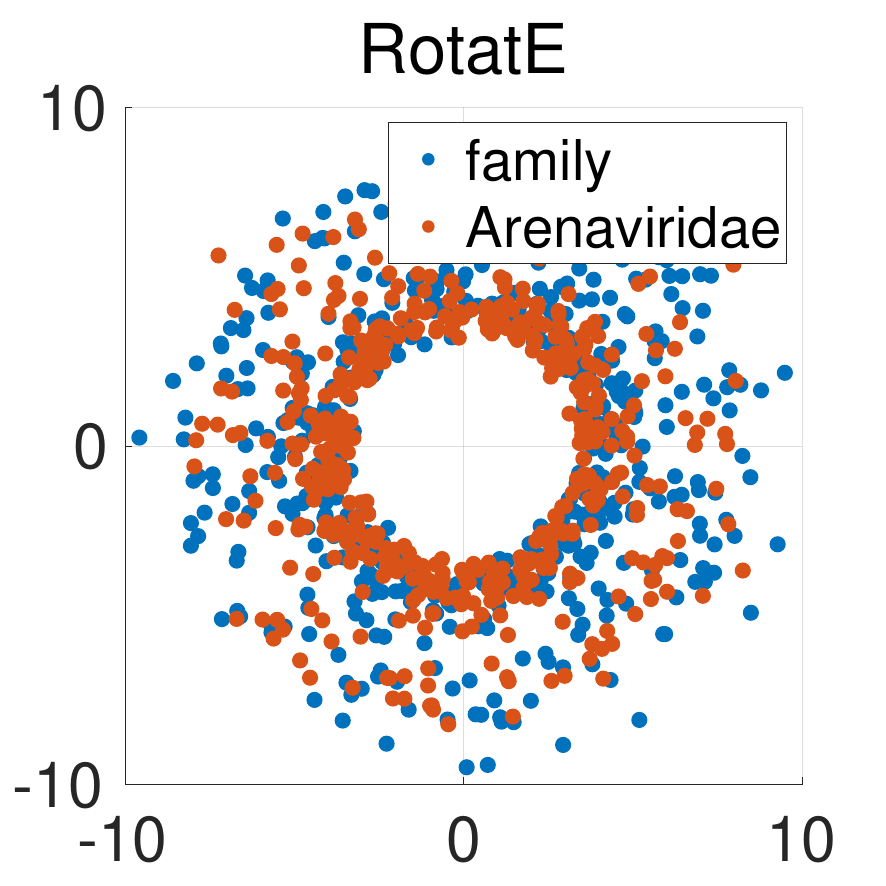}\hspace{4mm}
  \includegraphics[width=93pt]{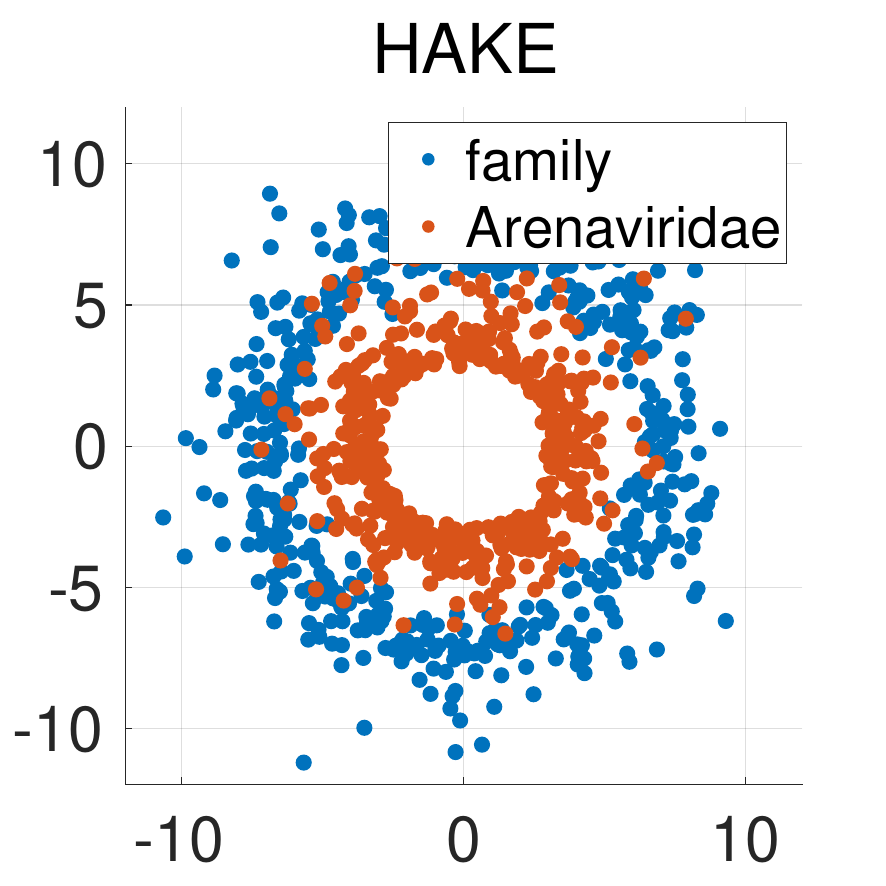}
  \caption{\textit{(family, \_member\_meronym, Arenaviridae)}}
  \label{fig:scatter_modulus_sub3}
\end{subfigure} 
\caption{Visualization of several entity embeddings from WN18RR dataset.}
\label{fig:scatter_modulus2}
\end{figure}

\end{document}